\documentclass[letterpaper, 10 pt, conference]{ieeeconf}  

\usepackage{flushend}

\usepackage{color}
\usepackage{xcolor}
\let\labelindent\relax
\usepackage{enumitem}
\usepackage[colorlinks=true]{hyperref}
\usepackage{graphicx} 
\usepackage{amsmath} 
\usepackage{amssymb}  
 
\usepackage{amsthm}
\usepackage{booktabs}
\usepackage{mdframed}
\newtheorem{theorem}{Theorem}[section]

\theoremstyle{definition}
\newtheorem{defn}{Definition}[section]

\DeclareMathOperator*{\argmin}{argmin}

\title{\LARGE \bf
\textsc{GameChat}: Multi-LLM Dialogue for Safe, Agile, and Socially Optimal Multi-Agent Navigation in Constrained Environments 

}

\author{Vagul Mahadevan, Shangtong Zhang, and Rohan Chandra\\
{\small \texttt{\{dub5nq, xdm2bt, aar8xx\}@virginia.edu}}\\
{\small University of Virginia}\\
{\small Code, Videos, Proofs at \href{https://gamechat-uva.github.io/}{\textbf{https://gamechat-uva.github.io/}}}\\
}

\begin{document}

\maketitle
\thispagestyle{empty}
\pagestyle{empty}

\vspace{-20mm}
\begin{abstract}

Safe, agile, and socially compliant multi-robot navigation in cluttered and constrained environments remains a critical challenge. This is especially difficult with self-interested agents with unique, unknown priorities in decentralized settings, where there is no central authority to resolve conflicts induced by spatial symmetry. We address this challenge by proposing an intuitive, but very effective approach, \textsc{GameChat}, which facilitates safe, agile, and deadlock-free navigation for both cooperative and self-interested agents in cluttered environments. Key to our approach is the idea that agents should resolve conflicts on their own using natural language to communicate, much like humans.
We evaluate \textsc{GameChat} in simulated environments with doorways and intersections. The results show that 
even in the worst case, \textsc{GameChat} reduces the time for all agents to reach their goals by over 35\% from a naive baseline and by over 20\% from a state of the art baseline in the intersection scenario, while doubling the rate of ensuring the agent with a higher priority task reaches the goal first, from 50\% (equivalent to random chance) to 100\%. We also demonstrate how \textsc{GameChat} can be extended to more than two agents.

\end{abstract}

\section{INTRODUCTION}

It is challenging for multiple robots with different hidden priorities to plan safe, agile, and deadlock-free (situations where no robot can move toward its goal for a few seconds) trajectories in cluttered and constrained environments (Figure~\ref{coverimage}). First, in decentralized systems, we have no central authority that coordinates agents in a manner that deadlocks will be prevented or resolved. Second, with self-interested agents~\cite{witteveen2006multi} that have conflicting objectives, we must ensure that jerk-agent behavior (jerk agents may break previously agreed-upon consensuses or socially compliant protocols leading to unsafe behavior) is not incentivized. 
Third, symmetry between the agents (which occurs when they are the same distance from a shared collision point and have the same velocity) must be broken in a socially optimal manner. If both move forward at the same speed they will collide, and if neither moves, they will deadlock.

To handle coordination and symmetry in cluttered spaces, strategies like the right-hand rule enforce consistent, clockwise movement to resolve deadlocks~\cite{zhou2017fast}, mimicking human behavior. However, such rules assume agents follow them even when it is not in their self-interest. Game-theoretic approaches~\cite{le2022algames, bhatt2023efficient} address this by ensuring strategies form a Nash equilibrium (NE, where no agent can unilaterally change its strategy and increase its payoff). Still, these methods often require global knowledge of agents’ private costs, which is unrealistic, and they face ambiguity in symmetric scenarios due to multiple equilibria.


\begin{figure}[thpb]
      \centering
      
      \includegraphics[scale=.7]{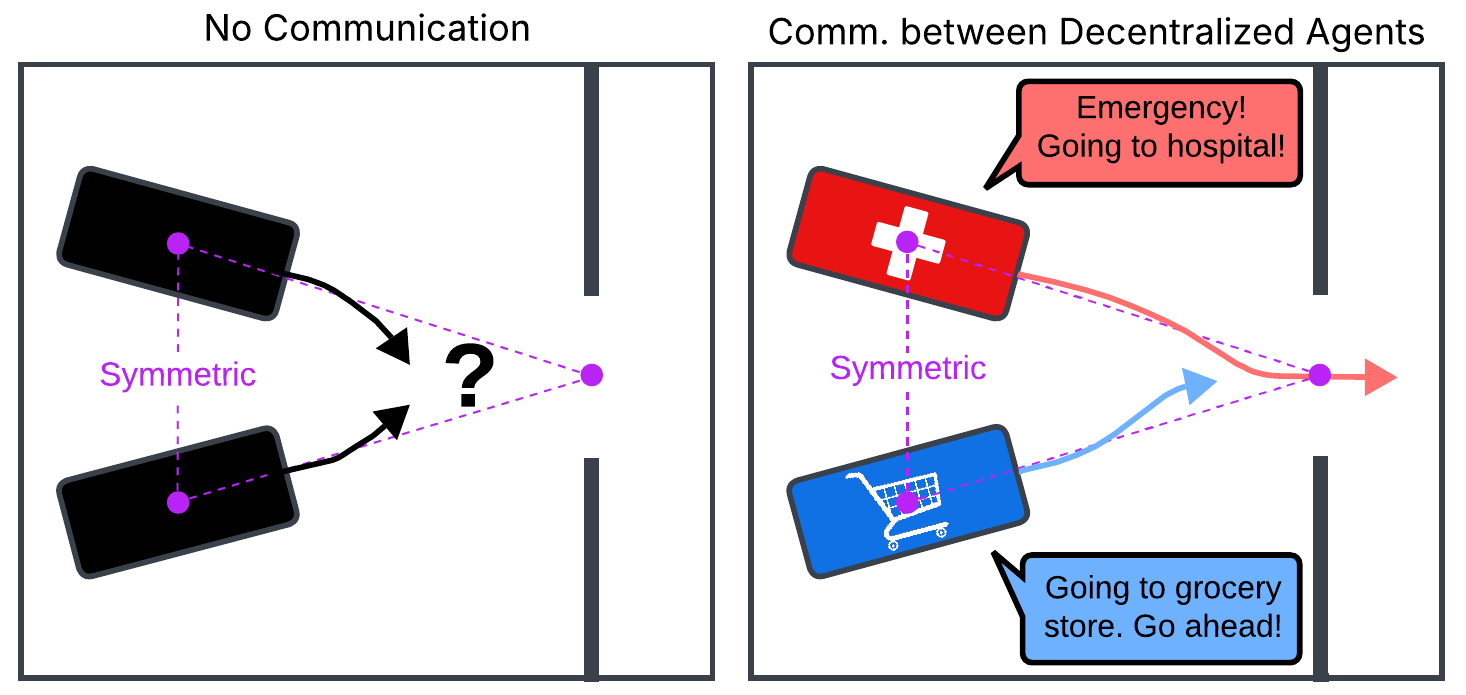}
      \caption{Two agents head toward a \textcolor{red}{\textbf{hospital}} and a \textcolor{blue}{\textbf{grocery store}} in a symmetric, constrained environment. In the left image, there is \textbf{no communication} between the agents, causing a deadlock as the agents do not know which should go first, and in the right image, \textsc{GameChat} uses natural language communication between decentralized agents to identify their roles, thereby resolving the deadlock by prioritizing urgent tasks.}
      \vspace{-15pt}
      \label{coverimage}
   \end{figure}

In symmetric scenarios, as one agent must slow down, a common approach is to define and maximize social welfare, a measure of how good a certain ordering of agents is. To do this, each agent must have some social priority, a measure of how important it is for that agent to reach the goal quickly. Previous work does this by assigning numerical priorities or budgets and using auctions to decide the order~\cite{suriyarachchi2022gameopt}, but these values are often arbitrary and require a centralized authority to manage the process. Instead, we suggest that priorities be grounded in the semantic tasks agents are pursuing. For example, we can agree it is more urgent for a first responder to reach a patient than for a tourist to get to the beach.


Since tasks are private (we do not broadcast our goals to the whole world at all times), we need a communication protocol for agents to decide who goes first. Humans use physical gestures and spoken language, but as many robots lack gestural capability, we focus on conversation. 


\subsection{Main Contributions}

We propose a new algorithm for safe, agile, deadlock-free, and decentralized multi-robot navigation in symmetric, constrained environments (like passing through doorways and narrow hallways). Our algorithm works for both cooperative and self-interested agents, that is, when agents choose to optimize their own objectives. Key to this approach is a novel LLM-based communication module in which agents automatically engage in a natural language dialogue with each other to proactively resolve any conflicts that could arise. Any conflict resolution is then executed via a low-level dynamic game-theoretic controller. Our method, which we call \textsc{GameChat}, has the following properties:


\begin{itemize}[noitemsep]
    \item \textbf{Subgame-Perfect Optimality:} \textsc{GameChat} includes a game-theoretic strategy which yields a subgame perfect equilibrium~\cite{selten1965spieltheoretische} (a guarantee that we have a Nash equilibrium now and at all future times), so self-interested agents will choose to commit to our strategy, as there is no incentive to deviate from it now or in the future.
    \item \textbf{Welfare-Maximizing (socially optimal):} \textsc{GameChat} is socially optimal, meaning that agents will maximize social welfare by prioritizing more urgent tasks. This also breaks spatial symmetry. 
    \item \textbf{Safety:} \textsc{GameChat} uses control barrier functions (CBFs)~\cite{ames2019control}, to guarantee safe trajectories.
    \item \textbf{Agility:} \textsc{GameChat} is minimally invasive, that is, the speeds of the agents are decreased as little as possible, and the agents do not spatially deviate from their desired paths, preserving smooth trajectories.
\end{itemize}


We choose natural language as our mode of communication, as it does not require a previously defined protocol (which may vary from one manufacturer to another). This also enables communication between robots and humans, a crucial step we must take as a community to achieve a future where humans and robots live and work together. 

\section{RELATED WORKS}


\noindent\textbf{Collision Avoidance:} To guarantee safety, one can use control barrier functions (CBFs)~\cite{ames2019control},~\cite{wang2017safety} which use the forward invariance of a set. If an agent is in a safe set at some time, then it will remain in that safe set for all future times. In MPC-CBF~\cite{zeng2021safety}, constraints are derived using CBFs and added to the receding-horizon controller to prevent collisions.


\noindent\textbf{Deadlock Resolution:} Symmetry between agents can result in deadlocks~\cite{alonso2013optimal}. 
The simplest way to break symmetry is to rely on the environment's randomness or to randomly perturb agents~\cite{wang2017safety}, but this can increase total cost. Heuristic-based approaches, such as a right-hand rule \cite{zhou2017fast}, \cite{chen2022recursive}  improve over random perturbations but are more centralized. 
Some methods design an auction mechanism \cite{suriyarachchi2022gameopt}, \cite{carlino2013auction}, \cite{chandra2023socialmapf}, and agents bid to cross the intersection using some bidding strategy. 
In reservation-based systems \cite{dresner2008multiagent}, agents must reserve slots to cross the intersection determined by the time to arrival.


\noindent\textbf{Learning-based Methods:} Deep reinforcement learning (DRL) has been used for multi-robot decentralized collision avoidance with local sensing \cite{long2018towards}. Inverse reinforcement learning (IRL) has been used to infer reward functions of other agents for planning in dense crowds \cite{chandra2024towards}. However, these methods have difficulty encoding safety constraints and in out-of-distribution domains. LiveNet \cite{gouru2024livenet} is a recent approach that encodes safety and agility via differentiable CBFs within a neural network controller. This allows for learnable multi-robot navigation in constrained spaces.


\noindent\textbf{Game-theoretic Methods:} For multi-agent planning with self-interested agents, methods have been developed for distributed optimization in differential games. These algorithms solve for Nash equilibria. 
Some are similar to direct methods in trajectory optimization \cite{le2022algames}, \cite{di2019newton} like Newton's method. They handle state and control input constraints and converge quickly. These algorithms yield analytical solutions that guarantee safety but not agility, and they require knowledge of the other agents' cost functions and policies/actions.


\noindent\textbf{Multi-Robot Communication:} In decentralized, cooperative multi-robot environments, communication methods are needed for coordination \cite{chandra2023deadlock}. 
Graph Neural Networks (GNNs) have been used to communicate the features of local observations among a network of robots \cite{li2020graph}. 
DMCA \cite{arul2022dmca} uses RL to learn selective communication to share goal information with relevant neighbors. Further, Arul et al. \cite{arul2024and} created an RL method that learns what information is important to communicate and when to communicate it, reducing indiscriminate broadcasting.


\noindent\textbf{LLMs for Robotics:} Recently, some methods use LLMs or Vision-Language Models (VLMs) in planning. Progprompt \cite{singh2023progprompt} uses LLMs to generate the code for robotic policies. On the multi-agent side, Garg et al. \cite{garg2024foundation} use LLMs and VLMs to resolve a deadlock after it occurs by identifying a leader agent to move first. One method \cite{song2024socially} uses VLMs to detect people and assign scores to trajectories, which induces socially compliant navigation. Most similar to our approach is RoCo \cite{mandi2024roco}, where robots are each equipped with an LLM instance and communicate in a natural language dialogue to agree on a plan of action. Note, however, that RoCo is fully cooperative while we investigate self-interested agents.

\section{PROBLEM FORMULATION}

In this section, we define a general problem formulation. We begin with a modified Partially Observable Stochastic Game (POSG) \cite{hansen2004dynamic}, defined by the tuple, $\left \langle k, T, \mathcal{X}, \{\mathcal{U}^i \}, \mathcal{T}, \Omega, \{\mathcal{O}^i \}, 
    \{\mathcal{R}^i \}
    \right\rangle$ where $k$ is the number of agents, $T$ is the finite number of time steps in the game, and $\mathcal{X}$ is the continuous state space. The superscript $i \in \{1, \ldots, k\}$ refers to the $i$th agent and the subscript $t \in \{0, \ldots, T\}$ refers to time step $t$; e.g., $\mathbf{x}^i_t \in \mathcal{X}$ is the state of agent $i$ at time $t$. Each agent $i$ has a start state $\mathbf{x}_0^i$ and a set of goal states $\mathcal{X}_g \subset \mathcal{X}$. The game ends if all agents have reached their goals or $T$ time steps have elapsed, whichever is sooner. For an agent $i$, $\mathcal{U}^i$ is the continuous control space containing feasible inputs. The dynamics function $\mathcal{T}: \mathcal{X} \times \mathcal{U}^i \rightarrow \mathcal{X}$ determines the state of an agent at time $t+1$ given its state and control input at time $t$. The set $\Omega$ is the observation space, and whenever agent $i$ arrives at a new state, the observation function $\mathcal{O}: \mathcal{X} \rightarrow \Omega$ yields a local observation of its own and nearby agents' states. We can define a trajectory of agent $i$ by $\Gamma^i = (\mathbf{x}_0^i, \ldots, \mathbf{x}_T^i)$ and an input sequence by $\Psi^i = (u_0^i, \ldots, u_{T-1}^i)$. 
For any time $t$, $C^i(\mathbf{x}_t^i) \subseteq \mathcal{X}$ is the space occupied by agent $i$ and two robots $i,j$ are in a collision if $C^i(\mathbf{x}_t^i) \cap C^j(\mathbf{x}_t^j) \neq \emptyset$. Each agent $i$ strives to maximize its payoff, $\mathcal{R}^i(\Gamma^i, \Gamma^{-i})$, which represents the total payoff at the end of the game for agent $i$ if they play $\Gamma^i$ and all other players take the trajectories in $\Gamma^{-i}$, where $\Gamma^{-i} = (\Gamma^1, \ldots, \Gamma^{i-1}, \Gamma^{i+1}, \ldots, \Gamma^{k})$.  Agents are rational:

\begin{defn}\label{def: rational}
    \textit{\textbf{Rationality:} An agent $i$ is rational if its payoff function $\mathcal{R}^i(\Gamma^i, \Gamma^{-i})$  satisfies two properties:}
    \begin{enumerate}
        \item \textit{Collision avoidance: $\mathcal{R}(\Gamma_1^i, \Gamma^{-i}) > \mathcal{R}(\Gamma_2^i, \Gamma^{-i})$ for a specific $\Gamma^{-i}$ if $\Gamma_2$ results in a collision but $\Gamma_1$ does not.}
        \item \textit{Time minimization: $\mathcal{R}(\Gamma_1^i, \Gamma^{-i}) > \mathcal{R}(\Gamma_2^i, \Gamma^{-i})$ for a specific $\Gamma^{-i}$ if $\Gamma_2$ results in higher time-to-goal than $\Gamma_1$, and neither trajectory results in a collision.}
    \end{enumerate}
\end{defn}


\begin{figure}[t]
      \centering
      
      \includegraphics[scale=.3]{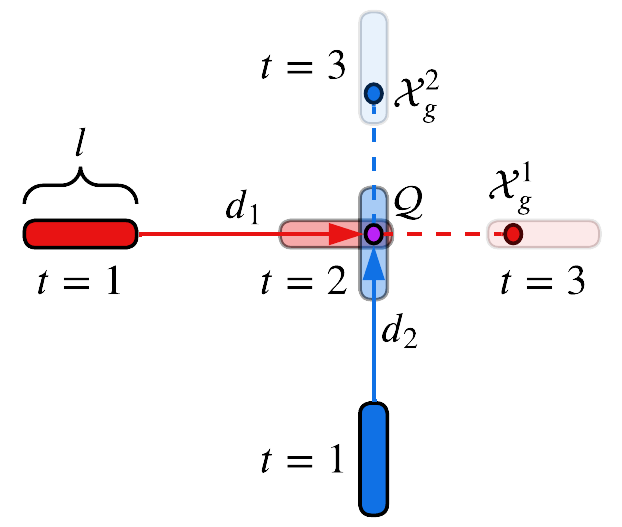}
      \caption{Example of a social mini-game. Each agent (represented by the red and blue rectangles) has length $l$ and must pass through $\mathcal{Q}$ on the way to their goals, $\mathcal{X}^1_g$ and $\mathcal{X}^2_g$. The desired trajectories $(\widetilde{\Gamma}^1, \widetilde{\Gamma}^2)$ for each agent are represented by the arrows. Note that they intersect at $\mathcal{Q}$ at $t=2$.}
      \label{smgex}
      \vspace{-10pt}
   \end{figure}

\begin{defn}\label{def: smg}
\textit{A \textbf{social mini-game} (SMG) is a type of POSG. Each agent has a desired trajectory $\widetilde{\Gamma}^i$ which is what the agent would follow if it were the only agent. The crucial property that makes an SMG is that there is some time $t$ and pair of distinct agents $i,j$ where $C^i(\mathbf{x}_t^i) \cap C^j(\mathbf{x}_t^j) \neq \emptyset$ and $\mathbf{x}_t^i \in \widetilde{\Gamma}^i, \mathbf{x}_t^j \in \widetilde{\Gamma}^j$; i.e., following the desired trajectories would cause the agents to collide at some point in time.} 
\end{defn}

We visualize an SMG in Figure~\ref{smgex}. Each agent has a straight-line desired trajectory to its goal. Both agents must pass through the collision point $\mathcal{Q}$, which models a doorway or the central point of a tight intersection. When both agents detect the existence of $\mathcal{S}$, each agent $i$ is a distance $d_i$ from $\mathcal{Q}$. We model them as thin, car-like objects of length $l$ and zero width. The agents share the velocity constraint $v_{\text{max}}$.
As a technical detail, for all times $t$ and for each agent $i$, $C^i(\mathbf{x}_t^i)$ does not contain the endpoints of the agent.

Additionally, each agent is assigned a social priority and we want to maximize social welfare ~\cite{harsanyi1955cardinal}, defined as:
\begin{defn}
    \textit{\textbf{Social welfare} is defined by the sum}
    {\small \begin{equation}
    \mathcal{W} = \sum_{i=1}^k \frac{p_i} {\tau^i\left(\Gamma^i\right)}
\end{equation}}
\textit{where $\tau^i(\Gamma^i)$ is the time-to-goal for agent $i$ if they take the trajectory $\Gamma^i$, and $p^i \in \mathbb{R}^+$ is the \textbf{social priority} of agent i, measuring how important it is for agent $i$ to reach its goal}.
\end{defn}

Note that if $p^i > p^j$ and $\tau^i(\Gamma^i) = \tau^j(\Gamma^j)$, reducing agent $i$'s time-to-goal would more positively impact social welfare than reducing agent $j$'s time-to-goal by the same amount. 

Next, agents will have to deviate from their desired trajectory to avoid collisions and deadlocks, but we also want to modify the desired trajectory as little as possible:

\begin{defn}\label{def: min inv}
    \textit{The modified trajectory ${\Gamma^{i,*}}$ is \textbf{minimally invasive} if it satisfies the following two properties:}
    \begin{enumerate}
        \item \textit{$\Delta \theta_t^i = 0$ for all $t$ (at all times, the heading of the robot does not deviate from the desired trajectory).}
        \item \textit{$ \min_t |v_t^i|$ is maximized (the robot slows down as little as is necessary to prevent a collision or deadlock).}
    \end{enumerate}
\end{defn}


\noindent Overall, given an SMG, our goal is to generate trajectories that are safe, deadlock-free, welfare-maximizing, and minimally invasive. We achieve this with \textsc{GameChat}.

\section{METHODOLOGY}
Here, we describe \textsc{GameChat} in detail. First, we discuss technical details of our environment and give an overview of the approach. Next, we present how we implemented the LLM dialogue between agents for prioritizing tasks.
Finally, we explain our game-theoretic control strategy and prove that it results in a subgame perfect equilibrium.

\subsection{Environment Details}

Our environment contains two agents running single-integrator unicycle dynamics. The state $(x, y, \theta) \in \mathcal{X}$ consists of 2D position and heading, and the control inputs $(v, \omega) \in \mathcal{U}^i$ consist of both linear and angular velocity. Each agent $i$ is trying to maximize its payoff $\mathcal{R}^i$. We implement this by defining a cost function $\mathcal{J}^i$ and using Model Predictive Control (MPC) to solve the following receding horizon optimization problem: 




{\small \begin{subequations}
\begin{align}
\left({\Gamma}^{i,*}, {\Psi}^{i,*} \right) &= \argmin_{(\Gamma^i, \Psi^i)} \sum_{t=0}^{T-1} \mathcal{J}^i(\mathbf{x}_t^i, u_t^i) + \mathcal{J}^i_T(x^i_T) \\
\text{s.t. } & \mathbf{x}_{t+1}^i = \mathcal{T}(\mathbf{x}_t^i, u_t^i), \forall t \in \{0,\ldots,T-1\} \\
&C^i(\mathbf{x}_t^i) \cap C^j(\mathbf{x}_t^j) \neq \emptyset, \forall t \label{eq: coll avoid} \\
&u_{\text{min}} \leq u_t^i \leq u_{\text{max}}, \forall t \label{eq: umax} \\
&\mathbf{x}_T^i \in \mathcal{X}^i_g.
\end{align}
\end{subequations}}

The rationality of the agent is captured through the collision avoidance constraint (\ref{eq: coll avoid}) and a penalty in the cost function for being far from the goal, incentivizing the agent to get to the goal as fast as possible.

At each time step, agent $i$ (though the function $\mathcal{O}^i$) observes the position, heading, and velocity of the other agent along with the position of the obstacles in the environment. The other agent and obstacles are treated as circular and control barrier functions are created for each one. Inequalities (analytically) generated from the CBFs are added in the next cycle of MPC as collision avoidance constraints (we refer the reader to \cite{chandra2023deadlock} for a full background on CBFs). 

      

\subsection{Technical Approach}

\begin{figure*}[t]
      \centering
      
      \includegraphics[scale=.2]{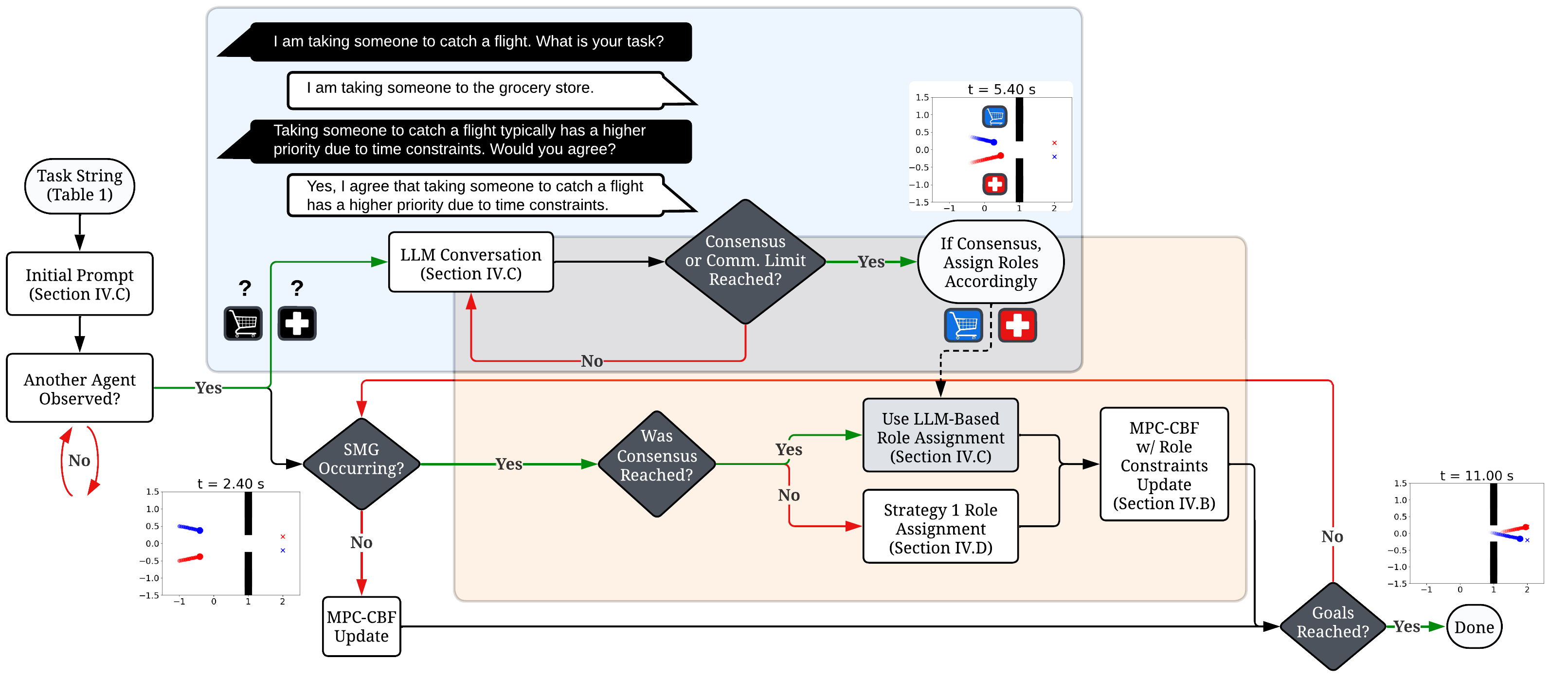}
      \caption{Flow chart describing the logical flow of \textsc{GameChat}. Blue box represents nodes involved in our novel LLM-based communication module. Orange box represents nodes handling a social mini-game (some LLM nodes are partially covered, representing that the agent may or may not be in those nodes during an SMG, as the communication could finish before an SMG begins or occur concurrently). See Figure \ref{smgex} for a visual example of an SMG.}
      \label{flowchart}
      \vspace{-15pt}
   \end{figure*}

Our technical approach is illustrated in Figure~\ref{flowchart}. At the earliest time $t$ when agents observe each other (i.e., $\mathbf{x}^j_t \in \mathcal{O}^i(\mathbf{x}^i_t)$ and $\mathbf{x}^i_t \in \mathcal{O}^j(\mathbf{x}^j_t)$), they begin exchanging messages until they reach a consensus or exhaust messages without agreement. Then, the agents assign themselves \textit{leader} and \textit{follower} roles, determining who will traverse $\mathcal{Q}$ first. If an SMG is detected before consensus or if communication fails, the agents default to Strategy 1 (defined in Section \ref{sec: GT strat}).

At each time step, agents check for SMGs—equivalent to imminent collisions—by casting rays from their positions in the direction of their headings. If these rays intersect and the agents' $\mathcal{Q}$ occupancy time intervals overlap, an SMG is present. If roles haven't been assigned, the agents default to Strategy 1 to assign them, potentially updating roles once consensus is reached. The leader proceeds at $v_{\text{max}}$, while the follower slows down to arrive just as the leader exits. This slowdown is enforced by updating the upper bound of the linear velocity constraint (\ref{eq: umax}) from $v_{\text{max}}$ to $\frac{d_i v_{\text{max}}} {l + d_j}$.

\begin{theorem}\label{thm: minimally invasive}
    \textsc{GameChat} yields minimally invasive trajectories (see Definition \ref{def: min inv}).
\end{theorem}

\begin{proof}
    The leader simply follows their desired trajectory, so there is no change in heading or velocity. The follower does not change their heading either. Their velocity is decreased by $v_{max} - \frac{d_i v_{\text{max}}} {l + d_j}$ until they reach $\mathcal{Q}$, but any less significant decrease in velocity would not prevent the collision. Therefore, the new trajectories are minimally invasive.
\end{proof}

\subsection{LLMs for Priority Determination}

Each agent has access to an LLM (\texttt{gpt-4o-mini}~\cite{hurst2024gpt}) instance through the OpenAI API and a task in natural language by way of an initial prompt:

\begin{mdframed}
\textit{You are taking someone to \textbf{TASK}. There is another agent taking someone to a location. You will have a conversation until you determine whether you have more or less priority as them depending on the tasks you and they are performing. Do not include pleasantries and be concise. Once you have reached a consensus with the other agent, output the number 1 and nothing else. Remembering your task correctly is paramount!}
\end{mdframed}

There are three types of agents. In decreasing order of priority they are: hospital agents, airport agents, and grocery agents (although we only test with three types of agents for simplicity, this could easily be generalized to an arbitrary number of agent types). Table~\ref{strtable} shows the possible task strings (which replace \textbf{TASK} in the initial prompt) corresponding to each type of agent.

\begin{table}[h]
\caption{Possible \textbf{TASK} Strings for Different Agent Types}\label{strtable}
\label{table_example}
\begin{tabular}{lcr}
\toprule
Hospital & Airport & Grocery\\
\midrule
``the hospital" & ``the airport" & ``the grocery store"\\
``the emergency room" & ``catch a flight" & ``the supermarket"\\
``the operating room" & ``board a plane" & ``the store"\\
``the ER" & ``reach the airport" & ``go grocery shopping"\\
``get surgery" & ``go to the airport" & ``buy groceries"\\
\bottomrule
\end{tabular}
\vspace{-10pt}
\end{table}

When the agents observe each other, they take turns sending messages; they receive the last message from the other agent, add it to the dialogue history, and query the LLM (since we are using the OpenAI API, there is network latency, but each message is generated within one second) for a reply to send back. The robots keep moving around the environment while communicating and will stop conversing if they reach a consensus as to which agent's task has the higher priority or each robot has sent four messages with no agreement. An example conversation is shown in Figure~\ref{flowchart}. If the agents come to a consensus, knowing which agent's task has a higher priority gives a way to break symmetry. It also allows an agent with a slightly higher $\text{TTQ}$ (time to reach $\mathcal{Q}$) reach the goal first, maximizing social welfare.

\begin{table}[h]
\caption{Nuanced tasks. Bolded are what LLMs chose as high priority}\label{strtable}
\label{nuancedtasks}
\begin{tabular}{cc}
\toprule
Task 1 & Task 2\\
\midrule
``the hospital for a routine physical" & \textbf{``flight departing very soon"} \\
\textbf{``the hospital with sirens blaring"} & ``flight departing very soon" \\
``fire truck going to station" & \textbf{``ambulance to the ER"} \\
\textbf{``fire truck to apartment on fire"} & ``ambulance to the ER" \\
\bottomrule
\end{tabular}
\vspace{-10pt}
\end{table}

While the concept of hospital, airport, and grocery agents can be generalized to arbitrary task types, we wanted to go beyond testing with just three task types that have very clear-cut priorities. To push the limits of the reasoning ability of the LLMs, Table \ref{nuancedtasks} shows some more complex scenarios to compare, showing that the LLMs reason quite accurately about prioritizing tasks, including understanding that a fire emergency that threatens many lives takes precedence over an ambulance that takes one person to the ER.


\begin{theorem}\label{thm: welfare}
    In a symmetric social mini-game ($TTQ_i = TTQ_j$ and $\tau^i(\widetilde{\Gamma}^i) = \tau^j(\widetilde{\Gamma}^j)$) where, without loss of generality, $p^i > p^j$, the modified trajectories $\Gamma^{i,*}$ and $\Gamma^{j,*}$ generated by \textsc{GameChat} maximize social welfare.
\end{theorem}

\begin{proof}
    Whichever agent is designated the leader will have a time-to-goal of $\tau_1$ and the follower will have a time-to-goal of $\tau_2$, where $\tau_1 < \tau_2$. \textsc{GameChat} assigns the agent with the higher priority to be the leader, so $\tau^i({\Gamma}^{i,*}) = \tau_1$ and $\tau^j({\Gamma}^{j,*}) = \tau_2$ so the resulting social welfare will be $\mathcal{W}_{G} = \frac{p_i}{\tau_1} + \frac{p_j}{\tau_2}$. The alternative assignment would yield $\mathcal{W}_{A} = \frac{p_i}{\tau_2} + \frac{p_j}{\tau_1}$. Subtracting $\mathcal{W}_{G} - \mathcal{W}_A$ yields $\frac{p_i - p_j}{\tau_1} - \frac{p_i - p_j}{\tau_2}$, which is positive since $p_i - p_j > 0$ and $\tau_1 < \tau_2$. So, \textsc{GameChat} leads to the role assignment which maximizes social welfare.
\end{proof}

\subsection{Game Theoretic Strategy (Strategy 1)}\label{sec: GT strat}

If communication is not possible, the agents are not able to come to an agreement, or the agents end up in an SMG while they are still in the process of communicating (consensus has not yet been reached), we need a control strategy for the agents to fall back on. We describe our proposed strategy (hereafter called Strategy 1). If there is no collision on the rest of the desired trajectory $\widetilde{\Gamma}^i$, the agent sets the linear velocity to $v_\text{max}$. Let $\text{TTQ}_i$ denote the time agent $i$ would take to reach $\mathcal{Q}$ (for now, assume there is asymmetry so $\text{TTQ}_i \neq \text{TTQ}_j$). Agent $i$ checks if $\text{TTQ}_i < \text{TTQ}_j$ (in our environment, this amounts to checking if $\frac{d_i}{v_{\text{max}}} < \frac{d_j}{v_{\text{max}}}$). If so, then agent $i$ chooses $v_{\text{max}}$ as its linear velocity for the whole game, taking the leader role. If $\text{TTQ}_i > \text{TTQ}_j$, agent $i$ selects a velocity so it reaches $\mathcal{Q}$ at the instant that agent $j$ completely clears $\mathcal{Q}$. For our environment, this means agent $i$ will choose a linear velocity of $\frac{d_i v_{\text{max}}}{l + d_j}$ until reaching $\mathcal{Q}$, from which point it chooses $v_{\text{max}}$, taking the follower role.

\begin{theorem}\label{thm: spe}
If both agents follow Strategy 1 at all times, then we have a subgame perfect equilibrium \cite{selten1965spieltheoretische}.
\end{theorem}

\begin{proof}
    Without loss of generality, assume that at the start of the game, $\text{TTQ}_i < \text{TTQ}_j$. If both agents play our strategy, there will be no collision, so agents must reduce their time-to-goal to increase their payoff (as agents are rational, see Definition \ref{def: rational}). But agent $i$ is already moving toward its goal in a straight line at $v_{\text{max}}$ so there is nothing it can do to get there faster. Agent $j$ is not moving at maximum speed, but if it went any faster without changing direction, it would collide with agent $i$, something it is not willing to do. Changing direction also would not let it reach the goal any faster since it must pass through $\mathcal{Q}$ first, which cannot happen until agent $i$ has completely cleared it. Since no agent can reduce their payoff unilaterally, we have a Nash equilibrium.

    We assume that subgames with perfect symmetry ($\text{TTQ}_i = \text{TTQ}_j$) do not occur as in practice, due to numerical error in computation and stochasticity of the real-world, agents' states and inputs will be slightly perturbed and with probability 1, asymmetry is maintained. Since every subgame has the same structure as the original game and maintains asymmetry, it is also an NE for both agents to play our strategy in each subgame and so we have an SPE. 
\end{proof}

\begin{table*}[t]
\centering
\caption{Performance of the various control methods. \textsc{GameChat} (no LLM) outperforms the baselines in Min $v$ and all methods w.r.t. Makespan. Additionally, \textsc{GameChat} SMG Comm. only had a slightly higher makespan than the Pre-Comm. variant. }
\label{resulttable}
\resizebox{\linewidth}{!}{
\begin{tabular}{rccccccc}
\toprule
\multicolumn{8}{c}{\textbf{Doorway Scenario}} \\
\midrule
Method & $(\downarrow)$ \# Coll. & $(\downarrow)$ \# DLs & $(\uparrow)$ \% CP & $(\downarrow)$ Hi Pri. TTG (s) & $(\downarrow)$ Makespan (s) & 
$(\uparrow)$ Min $v$ (m/s) & $(\downarrow)$ $\Delta$ Path (m) \\
\midrule
MPC-CBF \cite{zeng2021safety} & 0 & 18 & N/A  & N/A & N/A & N/A & N/A \\
SMG-CBF \cite{chandra2023deadlock} & 0 & 0 & 50 & 11.300 $\pm$ 1.289 & 12.400 $\pm$ 0.873 & 0.118 $\pm$ 0.003 & 0.009 $\pm$ 0.002 \\
\textbf{\textsc{GameChat}} (no LLM) &\textbf{0} & \textbf{0} & 50 & 10.733 $\pm$ 0.566 & \textbf{11.267 $\pm$ 0.194} & 0.200 $\pm$ 0.088 & 0.010 $\pm$ 0.001 \\
\textsc{GameChat} Ground Truth & 0 & 0 & 100 & 10.400 $\pm$ 0.291 & 11.533 $\pm$ 0.194 & 0.227 $\pm$ 0.025 & 0.009 $\pm$ 0.001 \\
\textbf{\textsc{GameChat}} Pre-SMG Comm. & \textbf{0} &\textbf{0} & \textbf{100 }& \textbf{10.400 $\pm$ 0.291} & 11.533 $\pm$ 0.194 & 0.227 $\pm$ 0.025 & 0.009 $\pm$ 0.001 \\
\textbf{\textsc{GameChat}} SMG Comm. & \textbf{0} & \textbf{0} & \textbf{100} & 10.467 $\pm$ 0.388 & 11.600 $\pm$ 0.291 &\textbf{0.228 $\pm$ 0.030}& \textbf{0.001 $\pm$ 0.001}  \\
\midrule

\multicolumn{8}{c}{\textbf{Intersection Scenario}} \\
\midrule
Method & $(\downarrow)$ \# Coll. & $(\downarrow)$ \# DLs & $(\uparrow)$ \% CP & $(\downarrow)$ Hi Pri. TTG (s) & $(\downarrow)$ Makespan (s) & 
$(\uparrow)$ Min $v$ (m/s) & $(\downarrow)$ $\Delta$ Path (m)  \\
\midrule
MPC-CBF \cite{zeng2021safety} & 0 & 6 & 50  & 11.300 $\pm$ 1.149 & 12.400 $\pm$ 0.000  & 0.061 $\pm$ 0.000 & 0.007 $\pm$ 0.000 \\
SMG-CBF \cite{chandra2023deadlock} & 0 & 0 & 50 & 12.600 $\pm$ 2.478 & 15.000 $\pm$ 0.291 & 0.102 $\pm$ 0.000 & 0.001 $\pm$ 0.000 \\
\textbf{\textsc{GameChat}} (no LLM) & 0 & 0 & 50 & 10.800 $\pm$ 0.823 & \textbf{11.600 $\pm$ 0.000} & 0.249 $\pm$ 0.013 & 0.001 $\pm$ 0.000 \\
\textsc{GameChat} Ground Truth & 0 & 0 & 100 & 10.200 $\pm$ 0.291 & 11.800 $\pm$ 0.291 & 0.233 $\pm$ 0.020 & 0.001 $\pm$ 0.000 \\
\textbf{\textsc{GameChat}} Pre-SMG Comm. & \textbf{0} & \textbf{0} & \textbf{100} & \textbf{10.200 $\pm$ 0.291} & 11.800 $\pm$ 0.291 & 0.233 $\pm$ 0.020 & \textbf{0.001 $\pm$ 0.000} \\
\textbf{\textsc{GameChat}} SMG Comm. & \textbf{0} & \textbf{0} & \textbf{100} & 10.267 $\pm$ 0.388 & 11.867 $\pm$ 0.388 & \textbf{0.234 $\pm$ 0.024} & \textbf{0.001 $\pm$ 0.000}  \\
\bottomrule
\end{tabular}
}
\vspace{-10pt}
\end{table*}

\textit{Remark}: Note that even without our assumption maintaining asymmetry in subgames, Strategy 1 results in a Nash equilibrium. However, if we demand an SPE, undesirable Nash equilibria involving non-credible threats are culled, and the equilibrium induced by Strategy 1 remains. This gives us guarantees on future behavior as neither agent will have any incentive to deviate from this strategy in future subgames.

However, being an SPE distinguishes this equilibrium from other undesirable Nash equilibria by preventing non-credible threats resulting from future irrational behavior. For instance, another Nash equilibrium would be for the agent farther from $\mathcal{Q}$ to say, ``I will move toward the goal as fast as possible, so you must slow down if you do not want to collide with me," and take on the leader role while the closer agent takes on the follower role. Indeed, given that the farther agent plays the leader role, the closer agent would not move faster as then the farther agent would collide into it. However, trusting in the rationality of the farther agent, if the closer agent were to reach $\mathcal{Q}$ first we know that the farther agent would be forced to slow to avoid collision and so its threat is not credible--it is worse off if it follows through on the threat.

\textit{Remark}: Work in experimental economics \cite{rabin2006experimental} shows that people are willing to make small sacrifices for those in their vicinity even if it technically is against their interests to do so. Thus, we believe that our communication model between agents is reasonable since it allows for courteous behavior mimicking that exhibited by humans. The primary purpose of Strategy 1 is to serve as a competent backup plan for when communication does not work out for whatever reason.

\subsection{\textsc{GameChat} for $k > 2$ agents}

The most direct way to scale \textsc{GameChat} beyond two agents would be to have all pairs converse, but this requires $k-1$ simultaneous LLM dialogues, which is computationally impractical if using local LLMs and cause intransitivities in the ordering. Instead, we have each agent broadcast its task in natural language. Once all tasks are known, each agent broadcasts a ranking of tasks by priority, determined via an LLM query. This reduces the number of LLM calls to grow linearly with agents rather than quadratically.


We now have a classic problem of social choice theory: creating an aggregate ranking from individual preferences. We apply a simple pairwise comparison, a Condorcet method, favoring agents ranked above others most often. Although theoretically prone to Condorcet cycles, we observed none in practice due to consistent LLM outputs. Still, robust alternatives like Condorcet methods that handle cycles or Kemeny ranking or Borda count could be used. \cite{fishburn2015theory}.

\section{EXPERIMENTS}

There are three key questions we address. First, without communication, does Strategy 1 outperform baseline methods? 
Second, how fast are the LLMs and, when we incorporate communication, how often do the conclusions of the LLMs match up with the true priorities? Finally, how do the different ways that communication can unfold compare to each other and the non-communicative methods?

To address the first question, we compared \textsc{GameChat} (no LLM) with MPC-CBF \cite{zeng2021safety}, which uses CBFs to ensure safety, and SMG-CBF \cite{chandra2023deadlock}, which applies CBFs for both safety and liveness. For the second question, we evaluated against \textsc{GameChat} Ground Truth, where agents know each other’s true priority and instantly adopt the correct roles. This serves as the gold standard we aim to achieve. The third question considers timing differences between when LLM communication starts (on mutual observation) and when SMGs are triggered (on imminent collision). We tested both extremes: the best case (\textsc{GameChat} Pre-SMG Comm.; agents reach consensus before entering an SMG) and the worst case (\textsc{GameChat} SMG Comm.; agents enter an SMG upon observing each other and communicate while executing Strategy 1 until reaching consensus).

We tracked several metrics. First, we counted the number of collisions and deadlocks. We next examined metrics involving priority: the percent of scenarios where the higher-priority agent got to $\mathcal{Q}$ first and the average time to goal for the higher priority agent. Finally, we tracked the makespan, defined as the total duration of the scenario (also the slower agent's time to goal) and metrics involving invasiveness: the slower agent's minimum velocity prior to reaching $\mathcal{Q}$ (higher is better since it indicates less significant slowdown) and the average deviation from the desired straight-line paths.

\subsection{Simulation Details}

We implemented the simulations in Python using \texttt{do\_mpc}~\cite{fiedler2023mpc} (which uses CasADi~\cite{andersson2019casadi}) and Ipopt~\cite{wachter2006implementation}). For LLMs, we used the \texttt{openai} package to access \texttt{gpt-4o-mini}~\cite{hurst2024gpt} (since different models reasoned accurately about the task, we chose a smaller one that would run faster). Both robots shared a maximum velocity of $v_{\text{max}} = 0.3m/s$ and a maximum angular velocity of $\omega_{\text{max}} = 1.0rad/s$. Agents had radii of $0.1m$, the simulation step time was $0.2s$, and the maximum runtime was $15s$. There were two scenario types. In the doorway scenario, both agents must pass through a doorway of size $0.4m$. Both agents start $2m$ to the left of the doorway and $0.5m$ north or south of it, with goals on the opposite side. In the intersection scenario, the intersection the agents must pass through is $0.4m$ by $0.4m$. Both start $2m$ away from the center of the intersection and their goals are $1m$ away on the other side. We created 18 versions of each scenario. The scenarios as described previously are symmetric so we had variants where we moved exactly one agent $0.25m$ away from the goal. We also had every combination of priorities such that neither was the same (e.g. hospital and airport, grocery and airport, etc.) Each agent was randomly given one of the task strings from their type's set of possible strings.

\subsection{Results}
The metrics are displayed in Table~\ref{resulttable}. Note that MPC-CBF, despite being safe from collisions, frequently failed to complete the scenario since it lacks deadlock resolution capabilities. All other methods, however, were always able to avoid collisions and deadlocks, and they also had little path deviation, indicating smoothness. Also, we ran a hardcoded experiment (from here on referred to as the hardcoded baseline) where the first agent needed to reach $\mathcal{Q}$ before the other agent was permitted to begin moving. It obtains a makespan of $18.4s$ in both doorway and intersection environments.

\subsubsection{Noncommunicative Methods}

\begin{figure}[t]
      \centering
      
      \includegraphics[scale=.085]{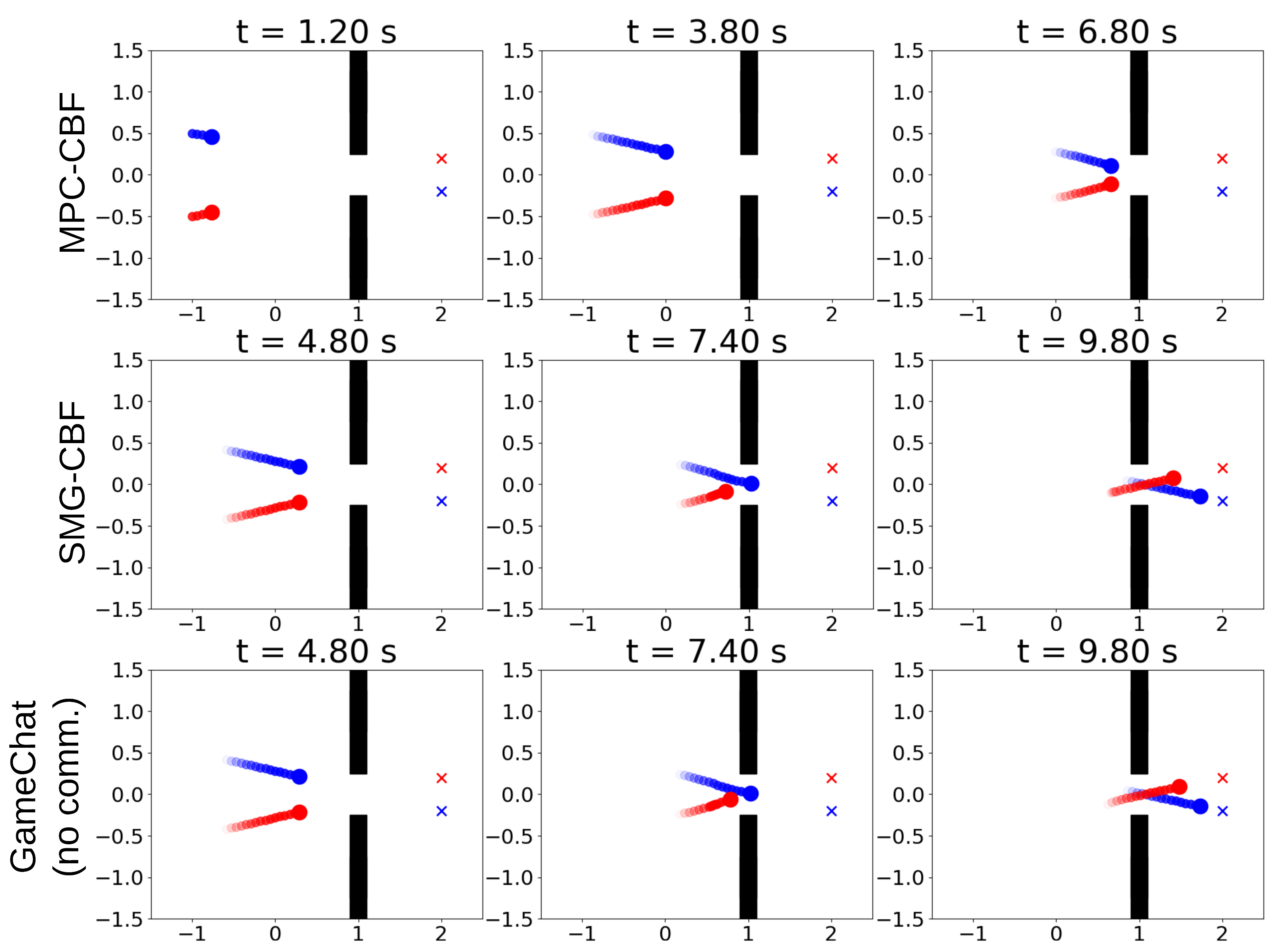}
      \caption{Trajectories generated by noncommunicative methods in the symmetric doorway. \textcolor{blue}{\textbf{Blue}} is a \textcolor{blue}{\textbf{grocery}} agent and \textcolor{red}{\textbf{red}} is a \textcolor{red}{\textbf{hospital}} agent. Top row is MPC-CBF, middle is SMG-CBF, bottom is \textsc{GameChat} (no LLM). MPC-CBF deadlocks (due to symmetry) and the other methods do not prioritize the hospital agent's urgent task (this happens 50\% of the time).}
      \label{doornocomm}
      \vspace{-10pt}
   \end{figure}

\begin{figure}[thpb]
      \centering
      
      \includegraphics[scale=.085]{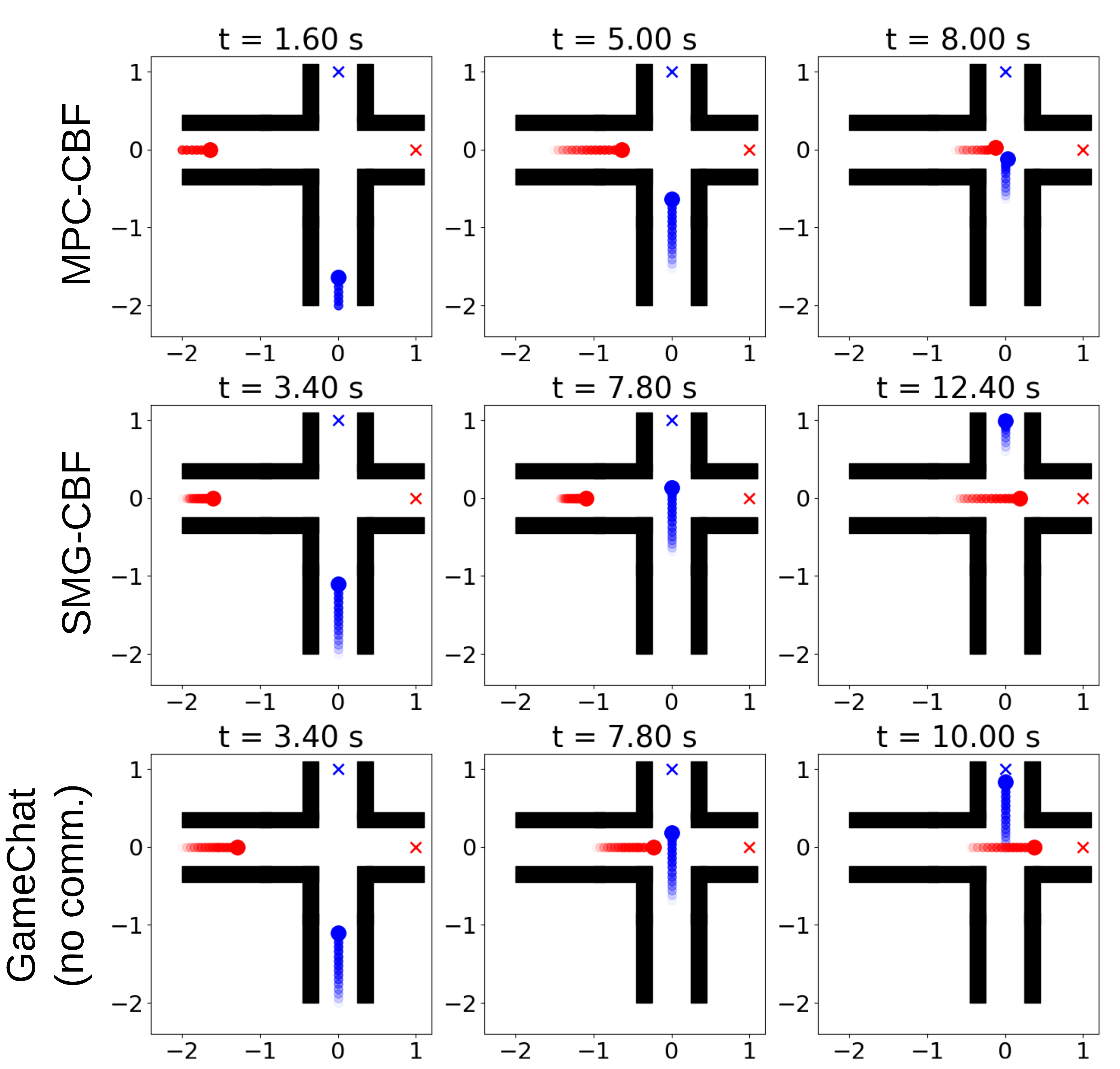}
      \caption{Trajectories generated by noncommunicative methods in symmetric intersection. \textcolor{blue}{\textbf{Blue}} is a \textcolor{blue}{\textbf{grocery}} agent and \textcolor{red}{\textbf{red}} is a \textcolor{red}{\textbf{hospital}} agent. Top row is MPC-CBF, middle is SMG-CBF, bottom is \textsc{GameChat} (no LLM). All methods do not prioritize the hospital agent's urgent task (this happens 50\% of the time). \textsc{GameChat} is much less invasive than SMG-CBF.}
      \label{internocomm}
      \vspace{-10pt}
   \end{figure}






\begin{figure}[t]
      \centering
      
      \includegraphics[scale=.083]{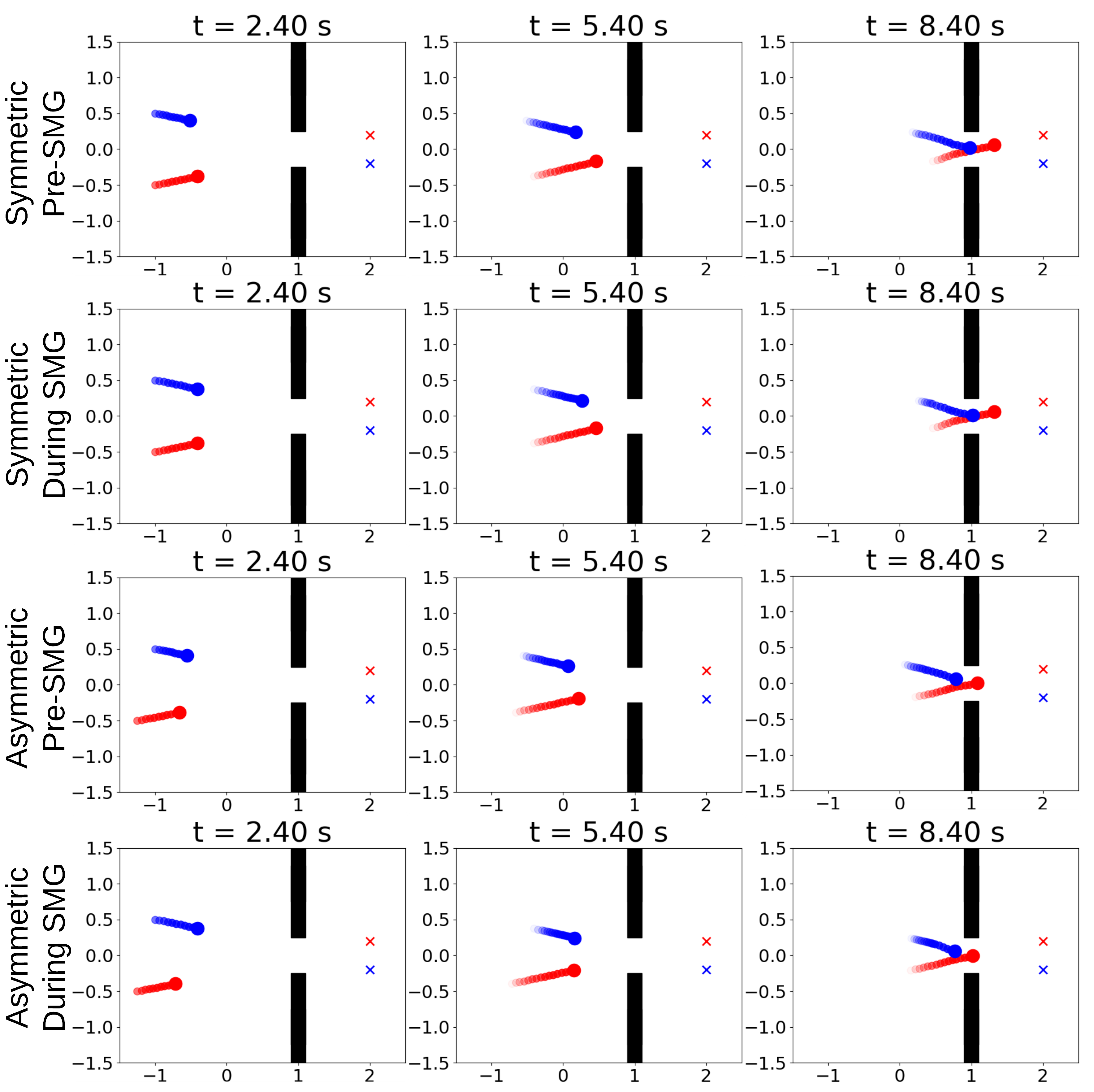}
      \caption{\textsc{GameChat} trajectories in doorway environment. \textcolor{blue}{\textbf{Blue}} is a \textcolor{blue}{\textbf{grocery}} agent and \textcolor{red}{\textbf{red}} is a \textcolor{red}{\textbf{hospital}} agent. Top row is pre-SMG conversation in the symmetric environment, second is during SMG conversation in symmetric, third is pre-SMG in asymmetric, and fourth is during SMG in asymmetric. The conversation finished at $t=2.86s$ for the symmetric (second row) and at $t=2.52s$ for the asymmetric (fourth row).}
      \label{doorcomm}
      \vspace{-15pt}
   \end{figure}

\begin{figure}[thpb]
      \centering
      
      \includegraphics[scale=.085]{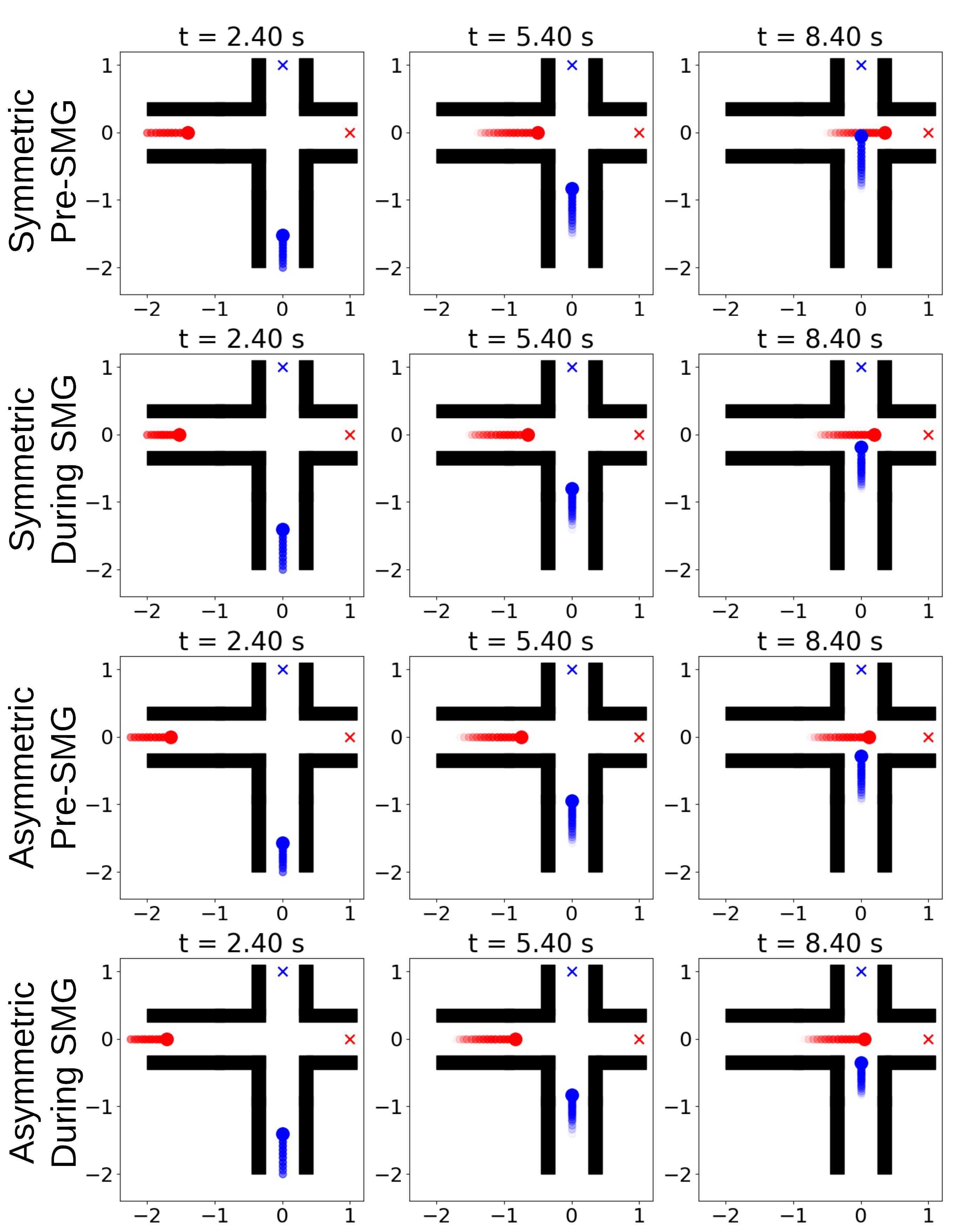}
      \caption{\textsc{GameChat} trajectories in intersection environment. \textcolor{blue}{\textbf{Blue}} is a \textcolor{blue}{\textbf{grocery}} agent and \textcolor{red}{\textbf{red}} is a \textcolor{red}{\textbf{hospital}} agent. Top row is pre-SMG conversation in symmetric environment, second is during SMG conversation in symmetric, third is pre-SMG with asymmetric, and fourth is during SMG with asymmetric. The conversation finished at $t=2.48$ for the symmetric (second row) and at $t=2.61$ for the asymmetric (fourth row).}
      \label{intercomm}
      \vspace{-10pt}
\end{figure}

We observe that \textsc{GameChat} (no LLM) is more agile than SMG-CBF, as it causes less deceleration in the second agent (see Min $v$ in Table~\ref{resulttable}). SMG-CBF's conservative behavior stems from its use of a parameter $\zeta$, requiring the slower agent to maintain $v_{\text{slow}} < \frac{1}{\zeta} v_{\text{fast}}$, regardless of each agent’s distance to $\mathcal{Q}$. This leads to the highest makespan among all methods, barely outperforming the hardcoded baseline in the intersection scenario. However, \textsc{GameChat} (no LLM) achieves the lowest makespan but ignores priority. In the symmetric doorway environment, it takes time to establish leader-follower roles (see Figure~\ref{doornocomm}), which increases makespan and reduces Min $v$, as one agent slows down greatly upon taking the follower role late. Still, our game-theoretic control strategy outperforms existing methods.

\begin{table*}[t]
\centering
\caption{ \textsc{GameChat} for more than two agents in the doorway scenario. }
\label{multiresulttable}
\resizebox{\linewidth}{!}{
\begin{tabular}{rccccccc}
\toprule

Method & $(\downarrow)$ \# Coll. & $(\downarrow)$ \# DLs & $(\uparrow)$ \% CP & $(\downarrow)$ Hi Pri. TTG (s) & $(\downarrow)$ Makespan (s) & 
$(\uparrow)$ Min $v$ (m/s) & $(\downarrow)$ $\Delta$ Path (m) \\
\midrule

\textsc{GameChat} 3-agent & 0 & 0 & 100 & 10.533 $\pm$ 0.394 & 12.533 $\pm$ 0.394 & 0.218 $\pm$ 0.000 & 0.038 $\pm$ 0.002 \\
\textsc{GameChat} 4-agent & 0 & 0 & 100 & 11.000 $\pm$ 0.362 & 13.533 $\pm$ 0.492 & 0.178 $\pm$ 0.000 & 0.075 $\pm$ 0.004 \\
\textsc{GameChat} 5-agent & 0 & 0 & 100 & 11.000 $\pm$ 0.319 & 14.567 $\pm$ 0.660 & 0.140 $\pm$ 0.012 & 0.059 $\pm$ 0.001 \\

\midrule

\end{tabular}
}
\vspace{-10pt}
\end{table*}

\subsubsection{LLM Performance, Prompt Robustness, and Welfare Maximization}
We measured the average duration of the LLM conversation as \textbf{2.767s} with a standard deviation of $0.441s$. This was short enough that in \textsc{GameChat} SMG Comm., the communication finished early enough that the new follower did not have to come to a stop to allow the new leader to pass (see Figure \ref{doorcomm}). For $k>2$, the average of the slowest of the $k$ queries was always \textbf{under one second}.

At first, the LLMs had difficulty with the task of conversing and determining priority (e.g., task forgetting, hallucinating false instructions, forgetting the consensus, etc.). But after carefully modifying the prompts, these problems became less and less frequent, to the point that \textsc{GameChat} performed as well as when agents were already assigned the correct roles (\textsc{GameChat} Ground Truth). The correct consensus was always reached since by default, LLMs are designed to be very honest and cooperative. Since \textsc{GameChat} was always able to reach the correct ordering of priorities in the consensuses, it maximizes social welfare (see Theorem \ref{thm: welfare}). Even in the asymmetric case where the higher priority agent starts farther away from $\mathcal{Q}$, \textsc{GameChat} is capable of having it go first, something the noncommunicative methods are not capable of (compare Figure \ref{doorcomm} and Figure \ref{doornocomm}).

\subsubsection{Pre-SMG vs. During SMG Communication, and Communicative vs. Noncommunicative Methods}

In both Pre-SMG and during SMG communication, \textsc{GameChat} identified the correct ordering of priorities. Since there is a delay in coming to consensus in \textsc{GameChat} SMG Comm., it has slightly higher makespans and higher priority time-to-goals than Pre-SMG Comm. However, these differences are quite small, showing that our LLM communication methodology is quite robust to worse-case scenarios. And interestingly, \textsc{GameChat} SMG Comm. had the lowest path deviation in the doorway scenario and the highest min $v$ in both environments. \textsc{GameChat} (no LLM) has a slightly lower makespan than the communicative methods. However, since it does not account for priorities at all, it (along with the other noncommunicative methods) only has a 50\% correct priority rate (the same as random guessing), while the communicative methods sacrifice a small amount of time in exchange for a doubling in accurately ordering the priorities, which is a worthwhile exchange since it greatly increases social welfare.

\begin{figure}[t]
      \centering
      
      \includegraphics[scale=.28]{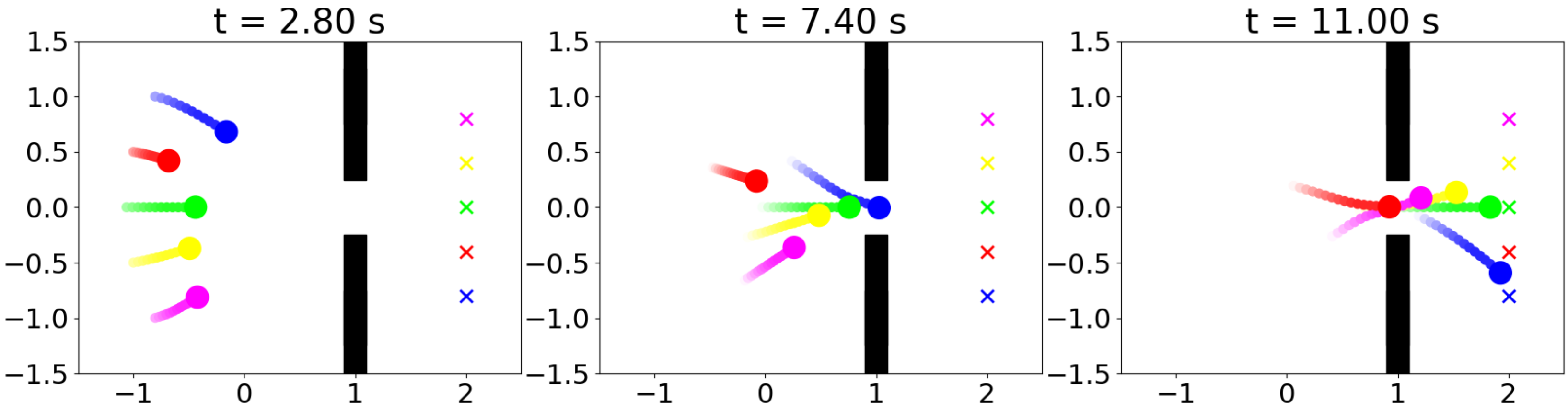}
      \caption{Five-agent \textsc{GameChat} trajectories. Blue and green are hospital agents, yellow and magenta are airport agents, and red is a grocery agent.}
      \label{pentaplot}
      \vspace{-15pt}
   \end{figure}

\vspace{3pt}

\subsubsection{Results for $k>2$ agents:} In Table \ref{multiresulttable}, we observe that \textsc{GameChat} effectively orders the tasks' priorities for 3, 4, and 5 agents. As the number of agents increases, however, it takes slightly longer for the slowest agent to reach its goal. This is because the slowest agent must reduce its velocity further to accommodate the greater number of agents passing by. Figure \ref{pentaplot} displays the trajectories from a run with five agents. The agents get as close to each other as is safe and order themselves correctly by priority.

\section{CONCLUSION, LIMITATIONS, FUTURE WORK}

We presented \textsc{GameChat}, a new approach for safe, agile, and socially optimal control in multi-robot constrained environments with self-interested agents. Agents communicate with each other by querying their LLMs to generate messages. Agents fall back on a game-theoretic strategy if they do not reach consensus. We demonstrate our approach's effectiveness in simulated constrained environments. Considering that we are using APIs and that the priority determination is meant to occur well before any conflict is imminent, \textbf{the latency of the LLM queries is quite low}.

A limitation is that LLMs tend to be honest and helpful, enabling cooperation, but also allowing dishonest agents to falsely claim higher priority. Future work includes addressing agent deception, fine-tuning LLMs (which performed well out-of-the-box), and using local models to reduce latency. While our environment used single-integrator dynamics, similar results apply to double-integrators, where the leader is the first agent unable to stop before $\mathcal{Q}$. Extensions include adapting \textsc{GameChat} to double-integrators or agents with high input costs, exploring more SMG scenarios (e.g., T-junctions, roundabouts), and bringing \textsc{GameChat} onto physical robots to validate real-world performance.

 \clearpage







{\footnotesize\bibliography{refs}
\bibliographystyle{ieeetr}}

\end{document}